\documentclass{article}

\usepackage{microtype}
\usepackage{graphicx}
\usepackage{subcaption}
\usepackage{booktabs} 

\usepackage{hyperref}

\usepackage[preprint]{icml2026}

\usepackage{amsmath}
\usepackage{amssymb}
\usepackage{mathtools}
\usepackage{amsthm}

\usepackage[capitalize,noabbrev]{cleveref}

\theoremstyle{plain}
\newtheorem{theorem}{Theorem}[section]

\theoremstyle{definition}

\newtheorem{assumption}[theorem]{Assumption}
\theoremstyle{remark}

\usepackage[textsize=tiny]{todonotes}

\icmltitlerunning{Fisher-Orthogonal Projected Natural Gradient Descent}

\begin{document}

\twocolumn[
  \icmltitle{Fisher-Orthogonal Projected Natural Gradient Descent for Continual Learning}



  \icmlsetsymbol{equal}{*}

  \begin{icmlauthorlist}
    \icmlauthor{Ishir Garg}{berk}
    \icmlauthor{Neel Kolhe}{berk}
    \icmlauthor{Andy Peng}{berk}
    \icmlauthor{Rohan Gopalam}{berk}
  \end{icmlauthorlist}

  \icmlaffiliation{berk}{University of California, Berkeley}
  \icmlcorrespondingauthor{Ishir Garg}{ishirgarg@berkeley.edu}
  
  \icmlkeywords{Machine Learning, ICML}

  \vskip 0.3in
]



\printAffiliationsAndNotice{}  

\begin{abstract}
Continual learning aims to enable neural networks to acquire new knowledge on sequential tasks. However, the key challenge in such settings is to learn new tasks without catastrophically forgetting previously learned tasks. We propose the Fisher-Orthogonal Projected Natural Gradient Descent (FOPNG) optimizer, which enforces Fisher-orthogonal constraints on parameter updates to preserve old task performance while learning new tasks. Unlike existing methods that operate in Euclidean parameter space, FOPNG projects gradients onto the Fisher-orthogonal complement of previous task gradients. This approach unifies natural gradient descent with orthogonal gradient methods within an information-geometric framework. We provide theoretical analysis deriving the projected update, describe efficient and practical implementations using the diagonal Fisher, and demonstrate strong results on standard continual learning benchmarks such as Permuted-MNIST, Split-MNIST, Rotated-MNIST, Split-CIFAR10, and Split-CIFAR100. Our code is available at \url{https://github.com/ishirgarg/FOPNG}. 
\end{abstract}

\section{Introduction}

Continual learning addresses the fundamental challenge of enabling machine learning models to learn sequentially from non-stationary data streams without forgetting previously acquired knowledge---a problem known as catastrophic forgetting \citep{mccloskey1989catastrophic}. When a neural network is trained on a sequence of tasks, gradient updates that improve performance on the current task typically degrade performance on previously learned tasks. The severity of catastrophic forgetting presents a significant barrier to deploying neural networks in realistic scenarios where data arrives sequentially and storage of all past data is impractical. Biological learning systems exhibit remarkable stability-plasticity trade-offs, maintaining existing knowledge while incorporating new information. Achieving similar capabilities in artificial neural networks remains an open challenge that has motivated substantial research.

Existing methods for continual learning fall into three broad categories: regularization-based methods that penalize changes to important parameters, replay-based methods that store and rehearse examples from previous tasks, and architecture-based methods that allocate dedicated parameters for each task. While these approaches have achieved varying degrees of success, they share a common limitation: they operate in Euclidean parameter space, which fails to account for the natural geometry of probabilistic models.

We propose the \textbf{Fisher-Orthogonal Projected Natural Gradient Descent (FOPNG)} optimizer, a geometrically principled optimizer that addresses catastrophic forgetting by enforcing orthogonality constraints in the Fisher-Riemannian manifold rather than Euclidean space. The key insight is that the Fisher information matrix provides a coordinate-invariant measure of how parameter changes affect model predictions. By projecting gradients onto the Fisher-orthogonal complement of previous task gradients, FOPNG reduces the impact that updates have on the output distributions of prior tasks while making efficient progress on the current task.

Our main contributions are: (1) a theoretical framework that unifies natural gradient descent with orthogonal gradient methods within information geometry, (2) a closed-form solution for the constrained optimization problem defining the projected update, and (3) practical algorithms using efficient Fisher approximations that demonstrate strong empirical results on continual-learning benchmarks.

\section{Background and Related Work}

In this section, we provide a comprehensive background on the foundational concepts underlying FOPNG and situate our work within the broader continual learning literature. We begin with the mathematical foundations of information geometry and natural gradients, then review existing approaches to continual learning, and finally discuss related work on efficient Fisher approximations.

\subsection{Information Geometry and the Fisher Information Matrix}

The Fisher information matrix lies at the heart of our approach, providing the geometric structure that enables principled continual learning. For a parametric model $p_\theta(y|x)$ with parameters $\theta \in \mathbb{R}^p$, the Fisher information matrix is
\begin{multline}
F(\theta) = \\\mathbb{E}_{x \sim q(x)}\mathbb{E}_{y \sim p_\theta(\cdot|x)}\nabla_\theta \log p_\theta(y|x) \nabla_\theta \log p_\theta(y|x)^\top
\end{multline}
where $q(x)$ is the data distribution \citep{amari1998natural}. Intuitively, it represents how fast a probability distribution changes at a point with respect to small perturbations in its parameter. The Fisher matrix has several important properties that make it fundamental to our approach.

The Fisher matrix defines a Riemannian metric on the parameter manifold, turning the space of probability distributions into a curved geometric space. Under this metric, the infinitesimal distance between $p_\theta$ and $p_{\theta+d\theta}$ is $ds^2 = d\theta^\top F(\theta) d\theta$. This geometric perspective was developed extensively by Amari and others in the field of information geometry \citep{amari1998natural}.

A powerful property of the Fisher metric is that it is invariant under reparameterization.

\begin{theorem}[Reparameterization Invariance of Fisher Norm]
\label{thm:reparam_invariance}
Let $\phi = \phi(\theta)$ be a smooth bijection with Jacobian $J = \frac{\partial \phi}{\partial \theta}$. Let $F_\theta$ and $F_\phi$ denote the Fisher information matrices in the $\theta$ and $\phi$ parameterizations, respectively. Then for any tangent vector $\delta\theta$,
\begin{equation}
\|\delta\theta\|_{F_\theta} = \|\delta\phi\|_{F_\phi}
\end{equation}
where $\delta\phi = J \delta\theta$.
\end{theorem}

\begin{proof}
See Appendix \ref{sec:proof_reparam}.
\end{proof}

Intuitively, this means that optimization in the Fisher metric respects the intrinsic structure of the probability model rather than depending on arbitrary choices of parameterization. Euclidean gradient descent, on the contrary, is not reparameterization invariant and can behave very differently under coordinate changes.

Another key property of the Fisher information for FOPNG is that the Fisher matrix locally approximates the KL divergence between nearby distributions.

\begin{theorem}[Local Quadratic Approximation of KL Divergence]
\label{thm:kl_approximation}
For a small parameter step $v$, the KL divergence satisfies
\begin{equation}
\text{KL}(p_\theta \| p_{\theta+v}) = \frac{1}{2} v^\top F_\theta v + o(\|v\|^2).
\end{equation}
\end{theorem}

\begin{proof}
See Appendix \ref{sec:proof_kl}.
\end{proof}

This means that the Fisher norm $\|v\|_F = \sqrt{v^\top F v}$ directly quantifies the change in model predictions induced by a parameter update. Steps that are small in Fisher norm cause small changes in the output distribution, regardless of their Euclidean magnitude.

\subsection{Continual Learning: Problem Formulation and Existing Approaches}

In continual learning, we seek to learn a sequence of tasks $\{T_1, T_2, \ldots, T_n\}$ without catastrophic forgetting. At each stage, the model parameterized by $\theta$ must adapt to a new task $T_k$ with loss $L_k(\theta)$ while preserving performance on all previous tasks $T_1, \ldots, T_{k-1}$. Standard gradient descent on $L_k$ typically causes parameters to drift in directions that degrade previously learned representations.

The fundamental tension is between \emph{plasticity}---the ability to learn new information---and \emph{stability}---the preservation of old knowledge. Different methods approach this trade-off through different mechanisms, which we now review.

\subsection{Natural Gradient Descent}

Natural gradient descent exploits Fisher geometry for efficient optimization \citep{amari1998natural}. Rather than Euclidean-small steps, it takes Fisher-small steps:
\begin{equation}
\theta_{t+1} = \theta_t - \eta F(\theta_t)^{-1} \nabla_\theta L(\theta_t),
\end{equation}
derived by minimizing the loss's linear approximation subject to a KL divergence constraint. Natural gradient methods converge faster than standard gradient descent by accounting for loss landscape curvature \citep{martens2015optimizing}, and reparameterization invariance makes them robust to architecture choices affecting parameter scaling.

\subsection{Regularization-Based Methods}

Regularization methods penalize changes to task-important parameters. Elastic Weight Consolidation (EWC) \citep{kirkpatrick2017overcoming} adds:
\begin{equation}
\Omega(\theta) = \frac{\lambda}{2} \sum_{i} F_i (\theta_i - \theta_{i}^*)^2,
\end{equation}
where diagonal Fisher entries $F_i$ identify important parameters at learned parameters $\theta^*$. While pioneering Fisher information use in continual learning, EWC's diagonal approximation ignores correlations, and performance is highly sensitive to hyperparameter tuning \citep{ewc-tuning}. FOPNG achieves more stable performance across hyperparameters, requiring minimal tuning.

Synaptic Intelligence \citep{zenke2017continual} estimates parameter importance online during training. Memory Aware Synapses \citep{aljundi2018memory} learns importance unsupervised. These also share the regularization paradigm.

\subsection{Replay-Based Methods}

Replay methods store and rehearse previous task examples. Experience Replay maintains a buffer sampled during training \citep{chaudhry2019tiny}. Gradient Episodic Memory (GEM) \citep{lopezpaz2017gradient} projects gradients to prevent loss increases on stored examples. A-GEM \citep{chaudhry2018efficient} efficiently projects onto a single reference gradient.

These show strong performance when memory permits storing representative examples. GEM's constraint formulation relates to our work—FOPNG also constrains gradient updates, but in Fisher-geometric rather than Euclidean space. In particular, FOPNG deals with the continual-learning setup in which we do not have access to past task data (for example, in applications where data security is paramount), hence we compare FOPNG to baselines that also don't require storing the original training data across tasks.

\subsection{Orthogonal Gradient Methods}

Orthogonal Gradient Descent (OGD) \citep{farajtabar2020orthogonal} projects gradients onto the Euclidean orthogonal complement of previous task gradients. For gradient matrix $G = [g_1, \ldots, g_{k-1}]$:
\begin{equation}
g_{\text{proj}} = g - G(G^\top G)^{-1} G^\top g.
\end{equation}
This zeros components along previous gradients, ideally preventing interference.

While conceptually appealing, OGD inherits Euclidean geometry limitations: normal gradients do not capture the effect on the output distribution of a model. \citet{ong} combined natural gradients with OGD by first computing a projection against past gradients in \textit{Euclidean space} and pre-conditioning the resulting projection by the Fisher information. However, they found that performance decreased with this formulation. We show empirically that the projection must occur in \textit{Fisher space} to respect the natural gradient geometry. A more recent work by \citet{fop}  uses similar ideas about Fisher orthogonality, but derives a different optimizer for training with large batch sizes and does not discuss continual learning. Our work explicitly formulates continual learning in Fisher-geometric terms, deriving closed-form solutions that unify natural gradients with orthogonal projections.

\subsection{Efficient Fisher Approximations}
\label{sec:fisher_approx}

Computational challenges motivate efficient Fisher approximations crucial for scaling FOPNG.

\subsubsection{Diagonal Approximations}

Diagonal approximation retains only diagonal entries, reducing storage and computation usage. Prior methods have effectively used the  diagonal Fisher \citep{kirkpatrick2017overcoming}:
\begin{equation}
\widehat{F}_{\text{diag}}(\theta) = \frac{1}{|\mathcal{D}|} \sum_{(x,y) \in \mathcal{D}} \left[\nabla_\theta \log p_\theta(y|x)\right]^2,
\end{equation}
with element-wise squaring. 

\subsubsection{Empirical Fisher}
The empirical Fisher uses observed labels $\tilde{y}$ rather than sampling from $p_\theta(y|x)$:
\begin{equation}
\widehat{F}_{\text{emp}}(\theta) = \frac{1}{|\mathcal{D}|} \sum_{(x,\tilde{y}) \in \mathcal{D}} \nabla_\theta \log p_\theta(\tilde{y}|x) \nabla_\theta \log p_\theta(\tilde{y}|x)^\top.
\end{equation}
While \citet{kunstner2019limitations} note that it may sometimes perform poorly, the empirical Fisher remains practical for continual learning.

\subsubsection{Kronecker-Factored Approximations}

KFAC models Fisher as Kronecker products \citep{martens2015optimizing}. For layer weight $W$:
\begin{equation}
F_W \approx A \otimes G,
\end{equation}
where $A$ and $G$ are input activation and output gradient covariances. This structure enables efficient inversion while capturing curvature. EKFAC corrects eigenvalue errors \citep{laurent2018evaluation}. Other approaches include Shampoo \citep{gupta2018shampoo} and low-rank approximations \citep{grosse2016kronecker}.

The choice of approximation is critical for trading off accuracy, memory, and computation. Although we use the diagonal approximation for computational efficency, future work may explore the effect of these alternatives when applied to FOPNG.

\section{Theoretical Framework}

We now present the theoretical foundations of FOPNG. We develop the method from first principles, starting with the constrained optimization formulation and deriving the closed-form solution.

\subsection{Preliminaries and Notation}

Let $\theta \in \mathbb{R}^p$ denote model parameters and
$p_\theta(y \mid x)$ a smooth conditional distribution.
For a data distribution $q(x)$, define the Fisher information matrix
\begin{equation}
F
=
\mathbb{E}_{x \sim q}
\mathbb{E}_{y \sim p_\theta(\cdot \mid x)}
\big[
\nabla_\theta \log p_\theta(y \mid x)
\nabla_\theta \log p_\theta(y \mid x)^\top
\big].
\end{equation}

For any $u,v \in \mathbb{R}^p$, define the Fisher inner product and norm
\begin{equation}
\langle u, v \rangle_F := u^\top F v,
\qquad
\|u\|_F := \sqrt{u^\top F u}.
\end{equation}

Let $L_{\text{new}}(\theta)$ be the new-task loss and
$g = \nabla_\theta L_{\text{new}}(\theta)$ its gradient.
Let $G = [g_1, \ldots, g_m] \in \mathbb{R}^{p \times m}$ collect gradients
from previous tasks.
We denote by $F_{\text{new}}$ and $F_{\text{old}}$ the Fisher matrices
computed on new-task and previous-task data, respectively.

\begin{assumption}
\label{ass:regularity}
(Regularity Conditions)
\begin{enumerate}
    \item The score function $\nabla_\theta \log p_\theta(y \mid x)$ exists and has finite second moments, so that the Fisher information matrices are well-defined.
    \item $F_{\text{new}}$ is positive definite.
    \item $G$ has full column rank.
\end{enumerate}
\end{assumption}

In Section~\ref{sec:regularization} we show that Assumption~2 can be relaxed
using standard regularization techniques.
Assumption~3 holds in practice since the number of stored gradients is far
smaller than the parameter dimension.

\subsection{Geometric Motivation}

In information geometry, the natural gradient
$F_{\text{new}}^{-1} g$ represents the steepest descent direction for the
new-task loss under the Fisher metric.
Any learning rule that aims to be invariant to reparameterization should
therefore be expressed relative to this geometry.

We adopt the viewpoint that learning proceeds by selecting a Fisher-preconditioned direction that approximates the raw gradient
$g$ while respecting constraints imposed by previously learned tasks.
To formalize this idea, we introduce an auxiliary variable $u \in \mathbb{R}^p$
and define the actual parameter update for some scalar $c > 0$ as
\begin{equation}
v = c F_\text{new}^{-1}(g - u).
\end{equation}
The vector $u$ represents the component of the gradient that must be removed
or modified in order to satisfy constraints derived from prior tasks.

Crucially, we reason about $u$ through its Fisher-preconditioned form
$F_{\text{new}}^{-1} u$, which lives in the same geometric space as the
natural gradient.

\subsection{Problem Formulation}

FOPNG is defined by three principles, each motivated by Fisher geometry.

\paragraph{Natural-Gradient Approximation.}
We seek a Fisher-preconditioned direction $F_{\text{new}}^{-1} u$
that approximates the raw gradient $g$ as closely as possible.
This is captured by minimizing the Fisher norm of the discrepancy:
\begin{equation}
\min_u \| g - F_{\text{new}}^{-1} u \|_{F_{\text{new}}}^2.
\end{equation}
This objective measures how well the chosen direction aligns with the
natural gradient of the new task.

\paragraph{Compatibility with Previous Tasks.}
To limit interference with previously learned tasks, we require that the
Fisher-conditioned component of $u$ aligns with directions spanned by
stored gradients:
\begin{equation}
F_{\text{old}}^{-1} u \in \operatorname{Col}(G).
\end{equation}
Intuitively, we can think about $u$ as telling us how much of the previous task gradients are contained in the natural gradient update for $u$. FOPNG attempts to remove the old-task interfering component from the new task gradient.

\paragraph{Natural-Gradient Trust Region.}
Finally, we constrain the magnitude of the update in the Fisher metric of
the new task. For some learning rate $\eta > 0,$ we enforce
\begin{equation}
||v||_{F_\text{new}} = \eta. \label{eq:trust_region}
\end{equation}
This trust-region constraint bounds the KL divergence induced by the natural gradient update and helps control the learning rate across tasks. Empirically, we find that this is needed for learning stability.

\subsection{Constrained Optimization Problem}

Combining these principles yields the following constrained optimization
problem:
\begin{align}
\min_{u \in \mathbb{R}^p} \quad
& \| g - F_{\text{new}}^{-1} u \|_{F_{\text{new}}}^2
\label{eq:new_objective} \\
\text{s.t.} \quad
& F_{\text{old}}^{-1} u \in \operatorname{Col}(G),
\label{eq:col_constraint}
\end{align}
The resulting update is $v = c F_\text{new}^{-1}(g - u)$ where $c$ is chosen to satisfy Equation \ref{eq:trust_region}. 

\subsection{Deriving the FOPNG Update}

\begin{theorem}[FOPNG Update]
\label{thm:projected_update_main}
Under Assumption~\ref{ass:regularity}, the FOPNG update is
\begin{equation}
v^*
=
\eta
\frac{
F_{\text{new}}^{-1} P g
}{
\sqrt{
g^\top
P^\top F_\text{new}^{-1} P
g
}
},
\label{eq:fopng}
\end{equation}
where
\begin{equation}
P = I - F_\text{old}G (G^\top F_\text{old} F_\text{new}^{-1} F_\text{old} G)^{-1} G^\top F_\text{old},
\end{equation}
\end{theorem}

\begin{proof}
See Appendix~\ref{sec:proof_fopng}.
\end{proof}

The matrix $P$ projects onto the subspace that is Fisher-orthogonal to
previous-task gradients when expressed in the geometry of the new task.
The final scaling enforces the natural-gradient trust region.

\subsection{Fisher Natural Gradient (FNG): An Unconstrained Baseline}
\label{sec:fng}
In order to evaluate the effect of our projections, we derive a simpler natural-gradient update that removes the projection constraint.

If the compatibility constraint
$F_{\text{old}}^{-1} u \in \operatorname{Col}(G)$ is removed,
the problem reduces to selecting the best Fisher-conditioned approximation
to $g$ under the trust-region constraint.
This effectively yields the standard natural gradient descent \cite{amari1998natural}.

\begin{theorem}[FNG Closed-Form Solution]
\label{thm:fng}
Consider the constraints
\begin{align}
v_{\mathrm{FNG}} &= c F_{\mathrm{new}}^{-1}g
 \\
c ||v_{\mathrm{FNG}}||_{ F_{\mathrm{new}}} &= \eta.
\end{align}
The resulting update is
\[
v_{\mathrm{FNG}}
=
\eta
\frac{F_{\mathrm{new}}^{-1} g}
{\sqrt{g^\top F_{\mathrm{new}}^{-1} g}}
\]
\end{theorem}

\begin{proof}
See Appendix~\ref{sec:proof_fng}.
\end{proof}

\subsection{FOPNG-PreFisher: Another Variant}

Instead of maintaining a single aggregate Fisher matrix $F_{\text{old}}$,
we may store task-specific Fisher-weighted gradients.
For each previous task $i$, define $\tilde{g}_i = F_i g_i$ and collect them
as
\[
\tilde{G} = [F_1 G_1, \ldots, F_{k-1} G_{k-1}].
\]

The optimization problem becomes
\begin{align}
\min_{u \in \mathbb{R}^p} \quad
& \| g - F_{\text{new}}^{-1} u \|_{F_{\text{new}}}^2 \\
s.t. \quad & u \in \operatorname{Col}(\tilde{G})
\end{align}

The solution has the same form as \eqref{eq:fopng}, with
\begin{equation}
P = I - \tilde{G} (\tilde{G}^\top F_\text{new}^{-1} \tilde{G})^{-1} \tilde{G}^\top
\label{fopng_prefisher_P}
\end{equation}

This variant removes $F_{\text{old}}$ from the per-iteration update and
reduces the number of Fisher matrix--vector products, making it particularly
attractive for long task sequences. We call this variation the \textbf{FOPNG-PreFisher} update.

\section{Practical Improvements and Hyperparameters}

\subsection{Algorithm Description}

Algorithms \ref{alg:fopng} and \ref{alg:fopng_prefisher} present the complete FOPNG and FOPNG-PreFisher procedures respectively. The algorithm maintains a collection $G$ of gradients from previous tasks. For each new task, it computes both the new-task Fisher matrix $F_{\text{new}}$ and the old-task Fisher matrix $F_{\text{old}}$. The projection ensures Fisher-orthogonality while the normalization ensures bounded Fisher-norm steps.

FOPNG introduces two key hyperparameters:
\begin{itemize}
    \item $\lambda$ is a regularization parameter needed for numerical stability when inverting potentially ill-conditioned matrices
    \item $\alpha$ determines how much weight to give to the new Fisher as $F_\text{old}$ is maintained as a moving exponential average
\end{itemize}

We present ablations on these parameters in Appendix \ref{app:ablations}; in general we find that performance across these hyperparameters is relatively stable and do not require significant tuning in practice. 

\begin{algorithm}[htbp]
\caption{FOPNG Algorithm}
\label{alg:fopng}
\small
\setlength{\abovedisplayskip}{2pt}
\setlength{\belowdisplayskip}{2pt}
\begin{algorithmic}[1]
\REQUIRE Initial parameters $\theta_0$, task sequence $\{D_1,\ldots,D_T\}$, learning rate $\eta$, regularization $\lambda$, new Fisher weighting $\alpha$, number of stored gradients per task $k$, number of epochs $E$
\STATE Initialize gradient memory $G \leftarrow []$. Train on task $D_1$ using Adam or SGD
\STATE Compute Fisher $F_{old} \leftarrow \widehat{F}(D_1)$; Store $G \leftarrow [g_1,...,g_k]$
\FOR{task $t = 2, 3, \ldots, T$}
    \FOR{epoch $e = 1, \ldots, E$}
        \STATE Recompute new-task Fisher: $F_{\text{new}} \leftarrow \widehat{F}(D_t)$
        \FOR{all batches $(x, y)$ in $D_t$}
            \STATE Compute gradient $g \leftarrow \nabla_\theta L_t(\theta)$
            \STATE Compute $v$ according to Equation \ref{eq:fopng}
            \STATE Update parameters: $\theta \leftarrow \theta - \eta v^*$
        \ENDFOR
    \ENDFOR
    \STATE Store final gradients: $G \leftarrow [G, g_{k(t-1)+1},...,g_{kt}]$
    \STATE Update old Fisher: $F_{\text{old}} \leftarrow (1 - \alpha)\, F_{\text{old}} + \alpha\, F_{\text{new}}$
\ENDFOR
\end{algorithmic}
\end{algorithm}
\begin{algorithm}[htbp]
\caption{FOPNG-PreFisher Algorithm}
\label{alg:fopng_prefisher}
\small
\setlength{\abovedisplayskip}{2pt}
\setlength{\belowdisplayskip}{2pt}
\begin{algorithmic}[1]
\REQUIRE Initial parameters $\theta_0$, task sequence $\{D_1,\ldots,D_T\}$, learning rate $\eta$, regularization $\lambda$, number of stored gradients per task $k$, number of epochs $E$
\STATE Initialize $\tilde{G} \leftarrow []$; Train on task $D_1$ using Adam or SGD
\STATE Compute Fisher $F_1 \leftarrow \widehat{F}(D_1)$; Store $\tilde{g}_i \leftarrow F_1 g_i$, $\tilde{G} \leftarrow [\tilde{g}_1,...,\tilde{g}_k]$
\FOR{task $t = 2, 3, \ldots, T$}
    \FOR{epoch $e = 1, \ldots, E$}
        \STATE Recompute new-task Fisher: $F_{\text{new}} \leftarrow \widehat{F}(D_t)$
        \FOR{all batches $(x, y)$ in $D_t$}
            \STATE Compute gradient $g \leftarrow \nabla_\theta L_t(\theta)$
            \STATE Compute $v$ according to Equations \ref{eq:fopng}, \ref{fopng_prefisher_P}
            \STATE Update parameters: $\theta \leftarrow \theta - \eta v^*$
        \ENDFOR
    \ENDFOR
    \STATE Compute Fisher $F_t \leftarrow \widehat{F}(D_t)$; Store $\tilde{g}_i \leftarrow F_t g_i$; Update $\tilde{G} \leftarrow [\tilde{G}, \tilde{g}_{k(t-1)+1},...,\tilde{g}_{kt}]$
\ENDFOR
\end{algorithmic}
\end{algorithm}

\subsection{Regularization}
\label{sec:regularization}
In practice, the Fisher information and gradient matrices may not be well-conditioned; consequently, the updates can become numerically unstable when taking the matrix pseudo-inverse. In practice, we find that a regularization parameter $\lambda$ is need to stabilize the matrix inversion when $G^\top F_{\text{old}} F_{\text{new}}^{-1} F_{\text{old}} G$ is ill-conditioned. Note that we apply the regularizer twice, once to invert $(F_{\text{new}}$ and once to invert $A$ (see Algorithm \ref{alg:fopng}). Although this causes the Fisher information to deviate further from its true value, our results show that FOPNG is still effective at reducing forgetting.

\subsection{Gradient Storage}

FOPNG requires storing gradients from previous tasks to enforce Fisher-orthogonality constraints. We maintain a buffer that stores a set number of gradients per task. Specifically, after completing training on task $t$, we compute and store the gradient $g_t = \nabla_\theta L_t(\theta_t^*)$ evaluated at the final parameters $\theta_t^*$ for that task. Following the methodology of \citet{farajtabar2020orthogonal}, we use the gradient coming from the single logit corresponding to the ground-truth label. For a data point $x$ coming from the $k$-th class ($y_k = 1$), we aim to keep only $\nabla_\theta L_{t,k}(\theta_t^*)$ invariant. 

In practice, we found that storing 80 gradients per task was sufficient and use this for our experiments. \citet{farajtabar2020orthogonal} show that larger buffer sizes lead to small gains in accuracy, which is also consistent with our results. 

\subsection{Fisher Information Estimation}

Computing the exact Fisher information matrix $F(\theta) \in \mathbb{R}^{p \times p}$ is intractable for modern neural networks with millions of parameters. We employ the diagonal Fisher approximation reducing storage from $O(p^2)$ to $O(p).$ This is computed as
\begin{equation}
\widehat{F}_{\text{diag}}(\theta) = \frac{1}{|\mathcal{D}|} \sum_{(x,y) \in \mathcal{D}} \left[\nabla_\theta \log p_\theta(y|x)\right]^2,
\end{equation}
where the square is applied element-wise.

For FOPNG-PreFisher, an alternative is to use the empirical Fisher when multiplying the task gradients by $F_i$ after training on task $i$; since this is a one-time cost it is still computationally feasible. 

Empirically, we found no benefit from using the empirical Fisher for FOPNG-PreFisher thus we use the diagonal Fisher for all experiments unless specified otherwise.

For most machine learning tasks, the dataset may be too large to compute the entire Fisher matrix over. In practice, we only use a fixed-size random batch of the data every time the Fisher information is computed. Section \ref{sec:ablations} provides ablations showing the effect of varying this batch size.

\section{Experimental Results}

\subsection{Experimental Design}

\paragraph{Datasets.}
We evaluate on standard continual learning benchmarks based on the MNIST \cite{mnist} and CIFAR \cite{cifar} datasets chosen to test different aspects of the method:
\begin{itemize}
\item \textbf{Permuted-MNIST} \cite{kirkpatrick2017overcoming} (5 tasks): Each task applies a fixed random permutation to the images.
\item \textbf{Split-MNIST} \cite{splitmnist} (5 tasks): Each task contains 2 digit classes. We train on digits 0, 1 first, then 2, 3, and so on with the last task being digits 8, 9. Note that this is \textit{not} the multi-head form seen in some work, but rather uses a standard 10-output prediction head.
\item \textbf{Rotated-MNIST} (5 tasks): Each task rotates the images by a fixed rotation (0, 10, 20, 30, 40 degrees for tasks 1 through 5 respectively).
\item \textbf{Split-CIFAR10} \cite{synapticintel} (5 tasks): Each task contains 2 image classes. We train on classes 1, 2 first, then 3, 4, and so on. Note that this does use a multi-head architecture similar to prior work.
\item \textbf{Split-CIFAR100} \cite{synapticintel} (10 tasks): Each task contains 10 image classes out of 100 total. We train on classes 1-10 first, then 11-20, and so on. This also uses a multi-head architecture.
\end{itemize}

\paragraph{Model Architectures and Training Setup.}
We implement FOPNG in PyTorch. For MNIST-based benchmarks we use a simple MLP with 2 hidden layers of 100 units each, following \citet{farajtabar2020orthogonal}. Every layer except the final one uses ReLU activations and the loss is softmax cross-entropy. For CIFAR-based benchmarks we use a CNN with 4 convolutional layers (two blocks of Conv-ReLU-Conv-ReLU-MaxPool, 3→32→64 channels) followed by 2 fully-connected layers (4096→256→256) with ReLU and dropout (0.5) similar to that used by \citet{gupta2018shampoo}. Each task has a separate output head (256→2). All layers except the output use ReLU activations and the loss is softmax cross-entropy. Across all benchmarks we train each task for a fixed number of epochs with the same batch size of 10 to enable direct comparisons across methods.

\paragraph{First Task Training.} To establish a consistent, strong baseline on the first task (for which there is no previous task to "forget"), we train using standard optimizers without continual learning constraints. For MNIST datasets, we use SGD on the first task with the same learning rate as the optimizer; due to the simplicity of MNIST this is able to learn the first task with almost perfect accuracy. All methods then switch to their respective continual learning optimizers starting from task 2. For Split-CIFAR10, all optimizers use SGD with a learning rate of $1e-3$ on the first task to establish a consistent baseline on this more challenging dataset. For Split-CIFAR100, all optimizers use SGD with a learning rate of $1e-2$ on the first task.

\paragraph{Hyperparameter Selection} For all optimizers we conducing a learning rate sweep from $1e-5$ to $1e-1$. For EWC, FOPNG, and FOPNG-PreFisher, we also swept their respective $\lambda$ hyper-parameters; hyper-parameters were chosen to maximize final average accuracy across all tasks on a validation set. A more detailed description of what hyper-parameters were used for each experiment is given in Appendix \ref{app:hyperparameters} for reproducibility.

\paragraph{Evaluation} Following the methodology of \citet{chaudhry2018efficient}, we use the average accuracy on all trained tasks as our evaluation metric. In particular, the average accuracy after training on task $k$ is
\[A_k = \frac{1}{k}\sum_{i=1}^ka_{i,k}\]
where $a_{i,k}$ is the test-set accuracy on task $i$ after training on task $k$.

\subsection{Experimental Results}
\begin{figure}[htbp]
\setlength{\abovecaptionskip}{2pt}
\setlength{\belowcaptionskip}{0pt}
    \centering
    \begin{subfigure}{0.41\textwidth}
        \centering
        \includegraphics[width=\linewidth]{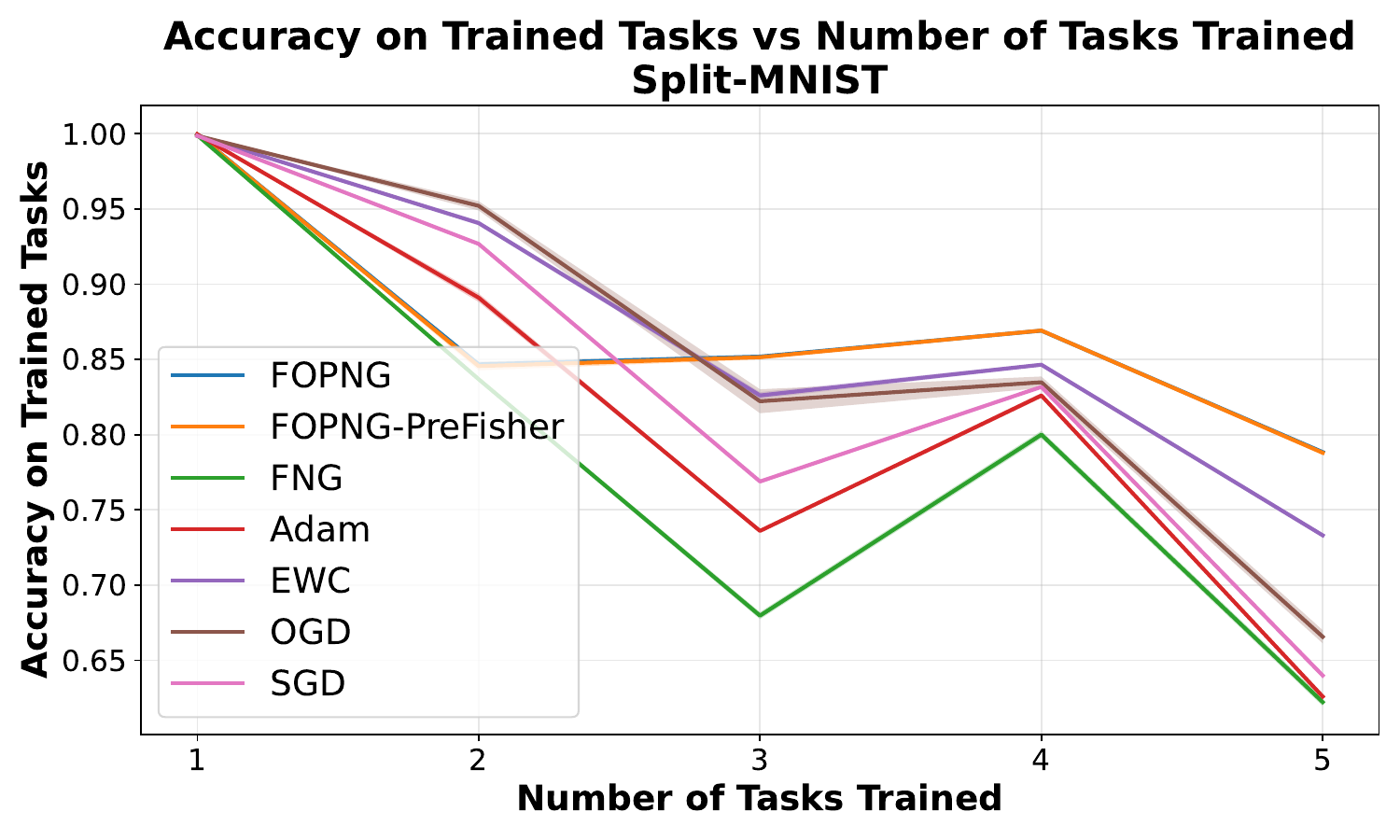}
        \label{fig:split_mnist_compare}
    \end{subfigure}\\[-1.5em]
    \begin{subfigure}{0.41\textwidth}
        \centering
        \includegraphics[width=\linewidth]{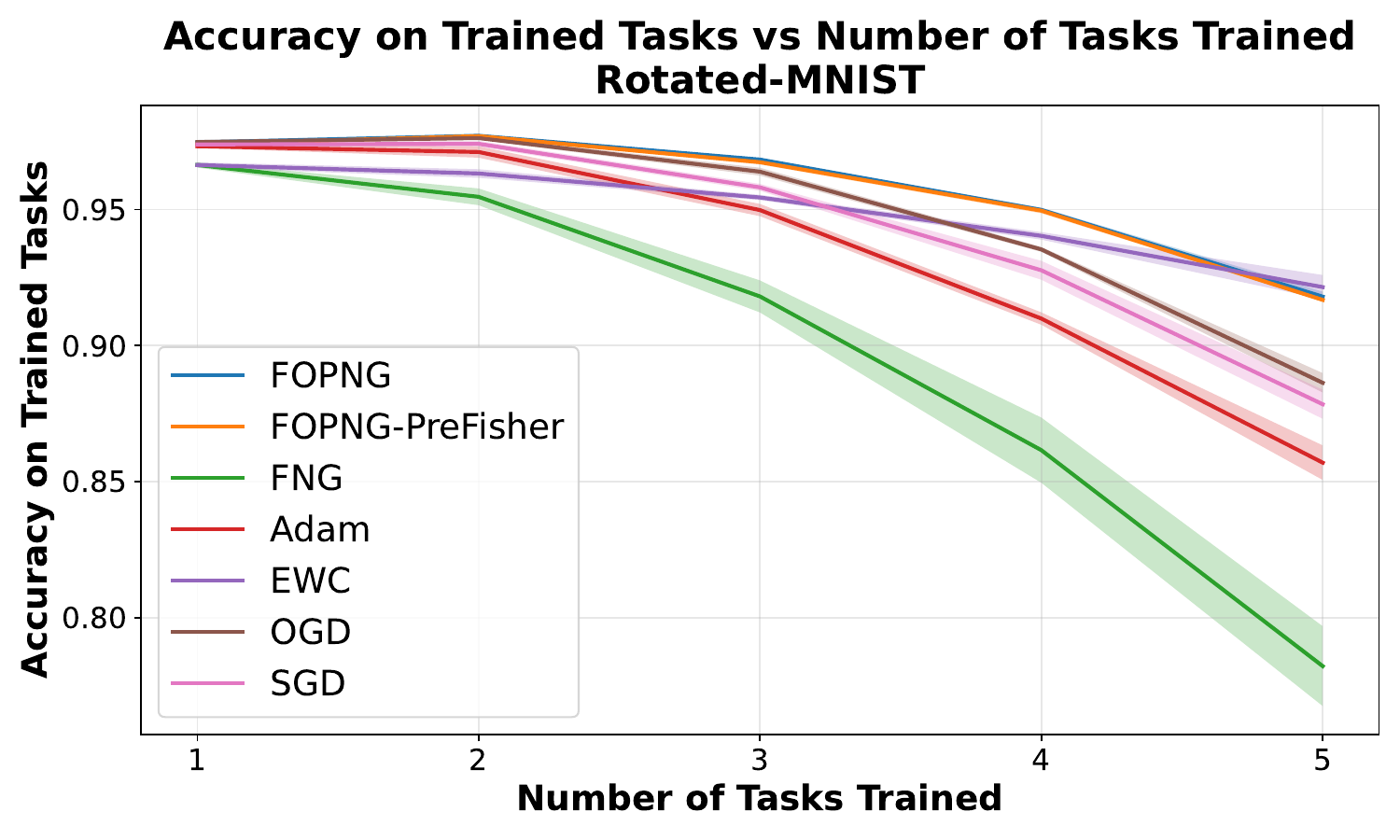}
        \label{fig:rotated_mnist_compare}
    \end{subfigure}\\[-1.5em]
    \begin{subfigure}{0.41\textwidth}
        \centering
        \includegraphics[width=\linewidth]{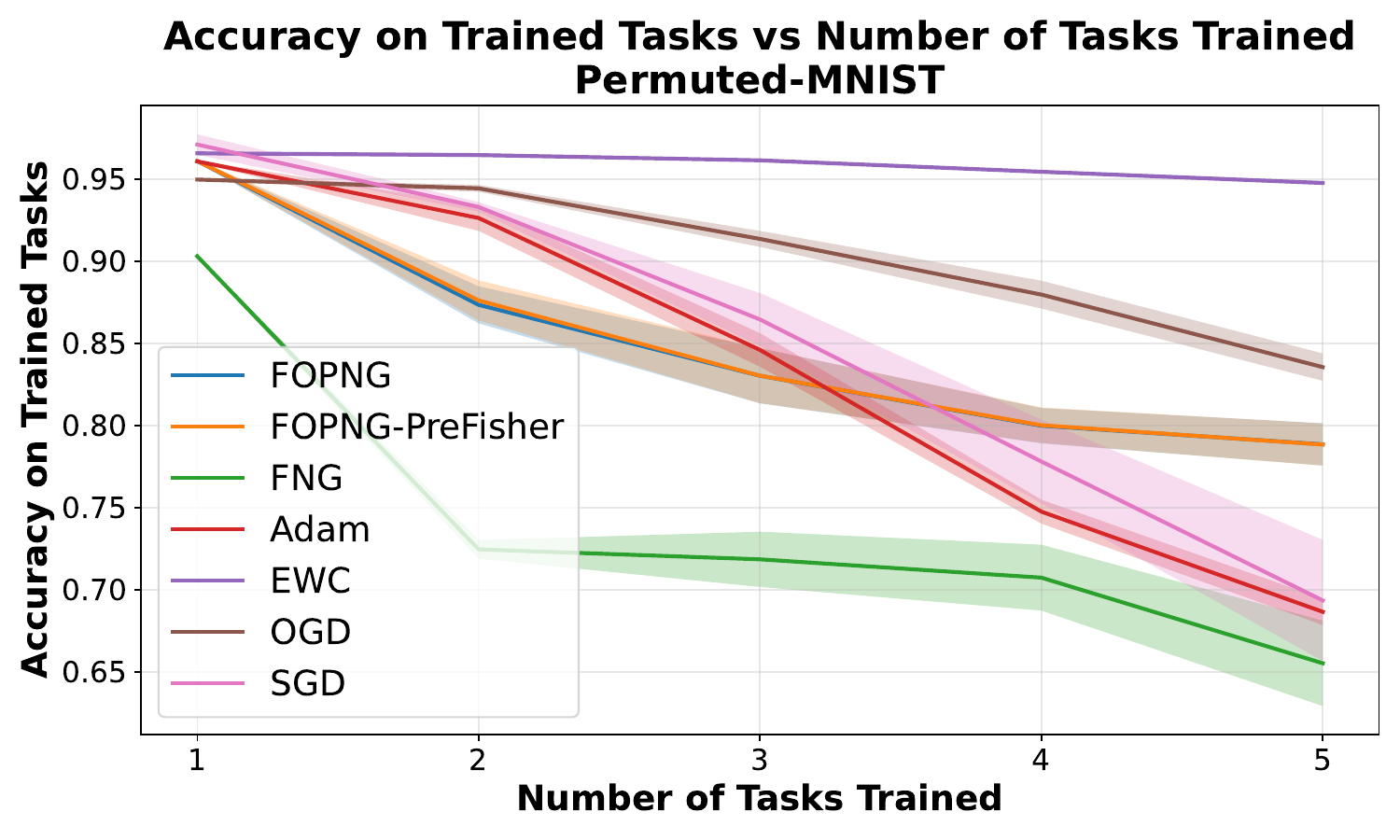}
        \label{fig:permuted_mnist_compare}
    \end{subfigure}\\[-1.5em]
    \begin{subfigure}{0.41\textwidth}
        \centering
        \includegraphics[width=\linewidth]{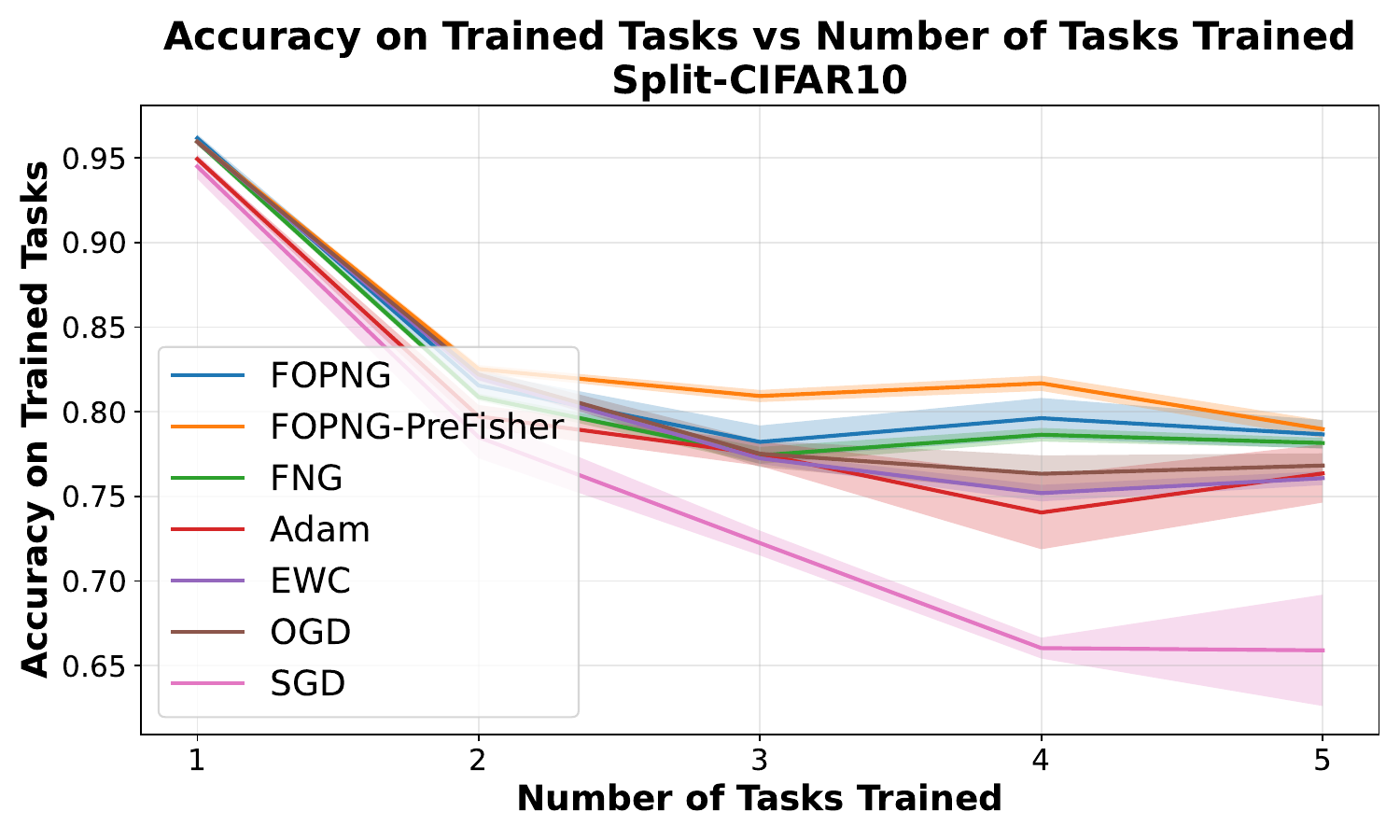}
        \label{fig:split_cifar10_compare}
    \end{subfigure}\\[-1.5em]
    \begin{subfigure}{0.41\textwidth}
        \centering
        \includegraphics[width=\linewidth]{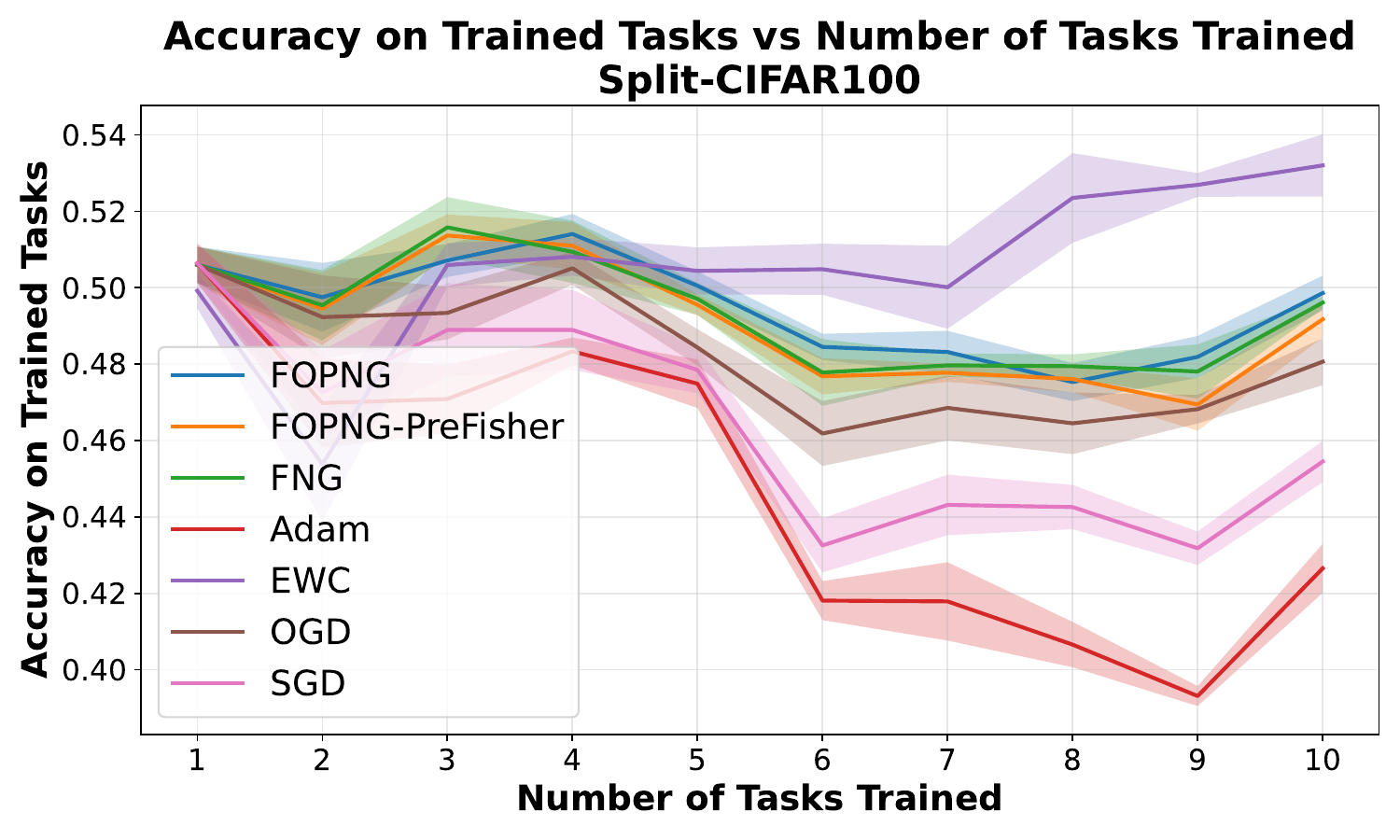}
        \label{fig:split_cifar100_compare}
    \end{subfigure}
    \caption{Comparisons across benchmarks. Shaded regions demonstrate 68\% confidence intervals across 5 independent training runs.}
    \label{fig:mnist_all}
\end{figure}

Figure \ref{fig:mnist_all} shows that FOPNG and FOPNG-PreFisher outperform all other optimizers on Rotated-MNIST, Split-MNIST, and Split-CIFAR10; with the exception that EWC performs comparably on Rotated-MNIST. We note that FOPNG performs poorly on Permuted-MNIST, possibly due to the benchmark creating highly out-of-distribution tasks in sequence \cite{robust-cl}. This suggests that FOPNG performs especially well on practical continual learning tasks where consecutive tasks are likely to have some distributional similarity. Importantly, we see that FOPNG outperforms OGD on all benchmarks (except Permuted-MNIST) by a significant margin, demonstrating that projections in Fisher space can preserve more information than Euclidean projections. Interestingly, EWC does very well on Split-CIFAR100, however, FOPNG still manages to outperform all other optimizers.

Another interesting point to note is that FOPNG and FOPNG-PreFisher exhibit very similar performance on all tasks. We also see that FNG performs poorly on most benchmarks, demonstrating that the Fisher-orthogonal projections are necessary for reducing forgetting.
\subsection{Ablation Studies}
\label{sec:ablations}

We conduct ablation studies on FOPNG's key hyperparameters: the regularization parameter $\lambda$ and the new Fisher weighting $\alpha$ in Appendix \ref{app:ablations}.

 Figure \ref{fig:lambda_ablation} shows that FOPNG performance is largely stable across a wide range of $\lambda$ values, with $\lambda \in [10^{-4}, 10^{-3}]$ being consistently good for all tasks. Very small values can lead to numerical instability while very large values begin to artifically perturb the Fisher geometry significantly. Figure \ref{fig:alpha_ablation} shows that FOPNG is remarkably robust to the choice of $\alpha$; this parameter effectively does not require tuning so we recommend 0.5 as a safe value to use. This stability is one key advantage of our method; unlike EWC, performance is much less sensitive to hyperparameter choice in FOPNG.

We also study the effect of changing the batch size used to compute the Fisher in FOPNG. Figure \ref{fig:fisher_bs_ablation} shows that there is a marginal gain in using batch sizes over 16. Remarkably, we find that even using just one random sample to estimate the Fisher information yields reasonable performance on average, just with higher variance. \citet{ewcfisher} shows that EWC requires large batch sizes for Fisher estimation to maintain performance; in contrast, FOPNG is significantly more stable and performs well for extremely small batch sizes demonstrating another strength of our method.

\begin{figure}[htbp]
    \centering
    \begin{minipage}{0.48\textwidth}
        \centering
        \includegraphics[width=\textwidth]{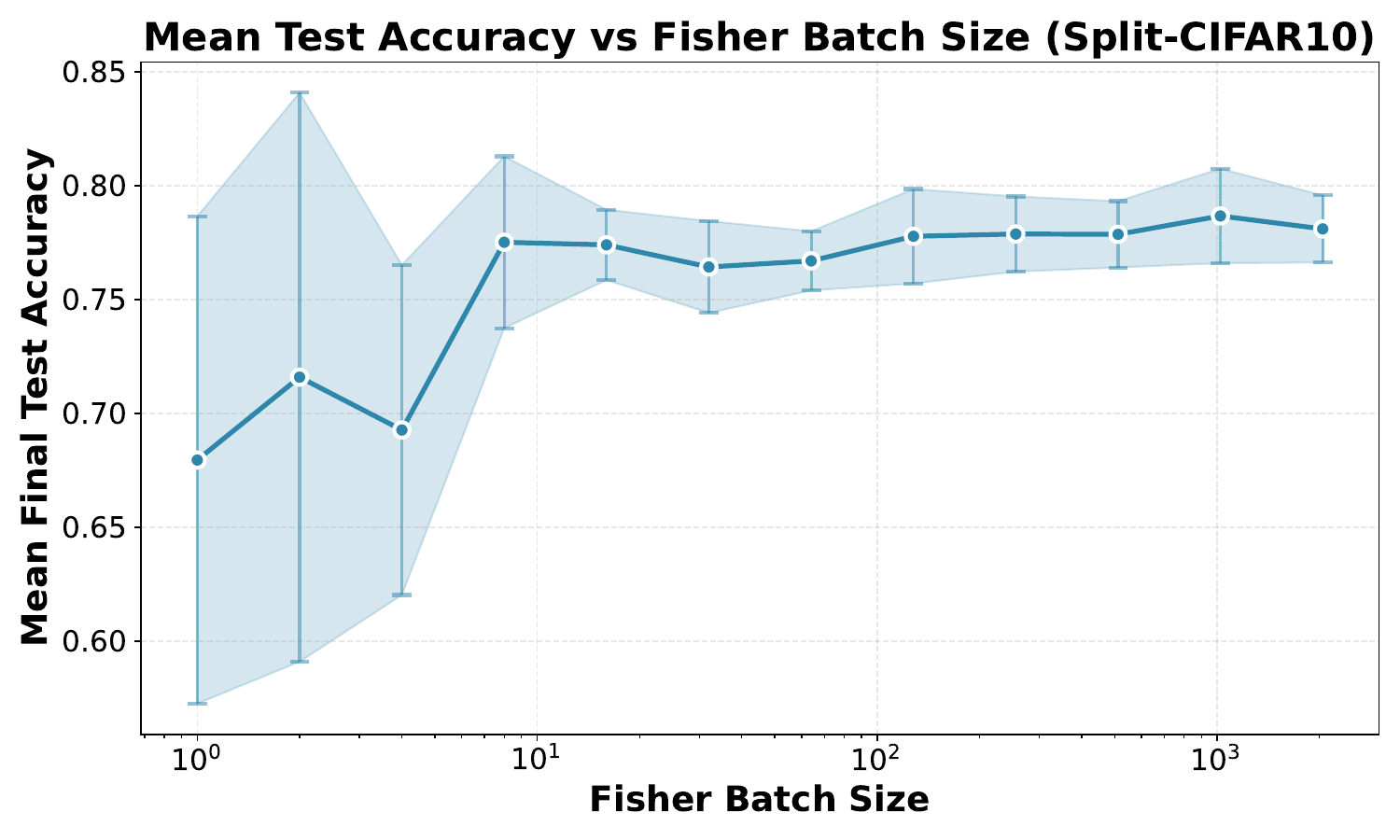}
    \end{minipage}
    \hfill
    \begin{minipage}{0.48\textwidth}
        \centering
        \includegraphics[width=\textwidth]{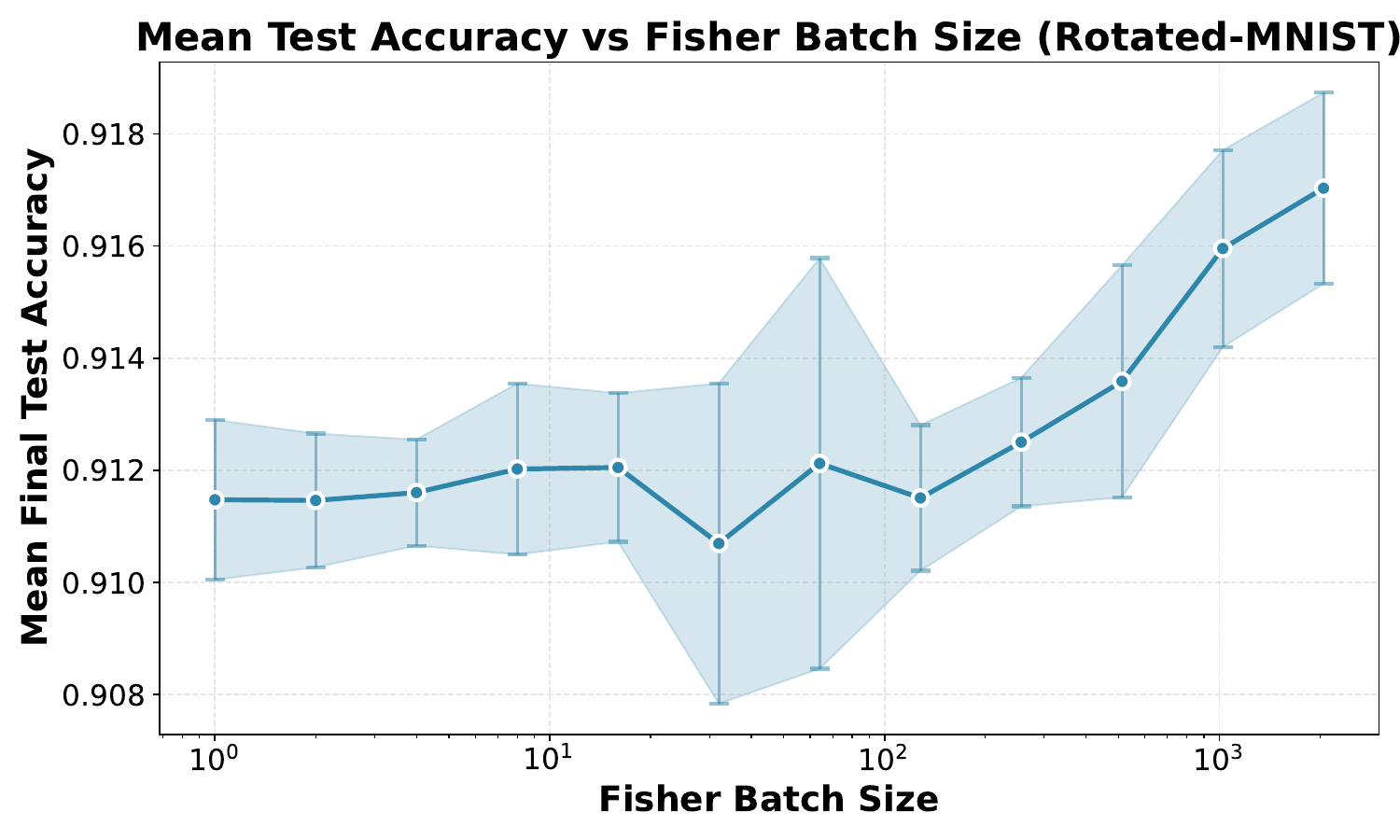}
    \end{minipage}
    \caption{Larger batch sizes increase performance and reduce variance. Shaded regions show 95\% confidence intervals across 5 independent runs.}
    \label{fig:fisher_bs_ablation}
\end{figure}

\subsection{Limitations}

We discuss the scope and limitations of FOPNG. Our current implementation assumes access to task-level gradients, which may limit applicability in settings with restricted memory or compute budgets. Additionally, one limitation of FOPNG is that training time scales with the number of continual tasks, making it less efficient for very long task sequences. In general, we find that on the datasets above, FOPNG requires around 40\% to 80\% more wall-clock time than EWC and OGD depending on the dataset. However, we believe that scalable implementations of FOPNG can significantly decrease this wall-clock time in a few ways:
\begin{itemize}
    \item Hyperparameter Stability: In general, FOPNG can achieve similar performance for a wider range of hyperparameters reducing the number of required hyperparameter sweeps.
    \item Software Optimizations: Our implementation is not optimized; FOPNG could potentially be implemented as a GPU kernel providing significant speedups.
    \item Framework Improvements: Estimating the Fisher requires computing per-sample gradients which is not efficiently supported by PyTorch (which we use); practical implementations can parallelize this much more efficiently.
\end{itemize}

Additionally, future work may improve on FOPNG to find ways to reduce the size of the gradient buffer or approximate it with a constant amount of space via low-rank approximations or other compression methods.

\section{Conclusion}

We have presented a geometrically principled optimizer for continual learning that enforces Fisher-orthogonality to previous task gradients while taking minimal Fisher-norm steps. The resulting direction is invariant under reparameterization, guarantees descent in the Fisher metric, and reduces changes in prior task outputs after training on new tasks. This approach unifies orthogonal and natural gradient methods within the Fisher-Riemannian geometry, providing a principled solution to the stability-plasticity dilemma in continual learning. We view our work as a mix of theoretical and empirical contributions that future work may build upon to create more salient deep learning optimizers. Notably, a key takeaway of our work is demonstrating that projections in Fisher-space are beneficial for reducing catastrophic forgetting.

\section{Impact Statement}

This paper presents work whose goal is to advance the field of machine learning. One potential impact of FOPNG is that it is specifically designed for scenarios where access to past training data is restricted. This is critical for applications where data security is paramount—such as healthcare or personalized edge computing—where models must learn from sensitive user data without needing to store or re-access it over time. We do not anticipate any significant negative societal impacts unique to this work, as it is a fundamental algorithmic improvement for training neural networks. However, like all advances in machine learning, its ultimate impact will depend on the specific applications for which these more capable continual learning systems are utilized.

\bibliography{main}
\bibliographystyle{icml2026}

\appendix
\section{Theoretical Derivations}
\label{sec:theory}

This appendix contains the complete proofs of all theorems presented in the main text.

\subsection{Proof of Theorem \ref{thm:reparam_invariance}: Reparameterization Invariance}
\label{sec:proof_reparam}

\begin{proof}
The Fisher information in the $\phi$-coordinates is
\begin{equation}
F_\phi(\phi) = \mathbb{E}\left[(\nabla_\phi \log p_\phi(y|x))(\nabla_\phi \log p_\phi(y|x))^\top\right].
\end{equation}
Under reparameterization, gradients transform as $\nabla_\phi = J^{-\top} \nabla_\theta$, so
\begin{align}
F_\phi &= \mathbb{E}\left[(J^{-\top} \nabla_\theta \log p_\theta)(J^{-\top} \nabla_\theta \log p_\theta)^\top\right] \nonumber\\
&= J^{-\top} F_\theta J^{-1}.
\end{align}
Now, for any tangent vector $\delta\theta$, letting $\delta\phi = J\delta\theta$,
\begin{align}
\|\delta\phi\|_{F_\phi}^2 &= \delta\phi^\top F_\phi \delta\phi \nonumber\\
&= (J \delta\theta)^\top (J^{-\top} F_\theta J^{-1})(J \delta\theta) \nonumber\\
&= \delta\theta^\top F_\theta \delta\theta \nonumber\\
&= \|\delta\theta\|_{F_\theta}^2.
\end{align}
\end{proof}

\subsection{Proof of Theorem \ref{thm:kl_approximation}: KL Divergence Approximation}
\label{sec:proof_kl}

\begin{proof}
Taylor expand the log-likelihood:
\begin{align}
\log p_{\theta+v}(y|x) &= \log p_\theta(y|x) + (\nabla_\theta \log p_\theta)^\top v \nonumber\\
&\quad + \frac{1}{2} v^\top H_\theta v + o(\|v\|^2),
\end{align}
where $H_\theta$ is the Hessian of $\log p_\theta$. The KL divergence is
\begin{align}
\text{KL}(p_\theta \| p_{\theta+v}) &= \mathbb{E}_{p_\theta}\left[\log p_\theta - \log p_{\theta+v}\right] \nonumber\\
&= -\mathbb{E}_{p_\theta}\left[(\nabla_\theta \log p_\theta)^\top v + \frac{1}{2} v^\top H_\theta v\right] \nonumber\\
&\quad + o(\|v\|^2).
\end{align}
Standard identities give $\mathbb{E}_{p_\theta}[\nabla_\theta \log p_\theta] = 0$ and $\mathbb{E}_{p_\theta}[-H_\theta] = F_\theta$, yielding
\begin{equation}
\text{KL}(p_\theta \| p_{\theta+v}) = \frac{1}{2} v^\top F_\theta v + o(\|v\|^2).
\end{equation}
\end{proof}

\subsection{Proof of Theorem \ref{thm:projected_update_main}: FOPNG Update}
\label{sec:proof_fopng}

\begin{proof}
We solve the constrained optimization problem
\begin{align}
\min_{u \in \mathbb{R}^p} \quad & \| g - F_{\text{new}}^{-1} u \|_{F_{\text{new}}}^2 \\
\text{s.t.} \quad & F_{\text{old}}^{-1} u \in \operatorname{Col}(G).
\end{align}

Expand the objective:
\begin{align}
\| g - F_{\text{new}}^{-1} u \|_{F_{\text{new}}}^2 &= (g - F_{\text{new}}^{-1}u)^\top F_{\text{new}} (g - F_{\text{new}}^{-1}u) \nonumber\\
&= g^\top F_{\text{new}} g - 2g^\top u + u^\top F_{\text{new}}^{-1} u.
\end{align}

The constraint $F_{\text{old}}^{-1} u \in \operatorname{Col}(G)$ means there exists $\alpha \in \mathbb{R}^m$ such that
\begin{equation}
F_{\text{old}}^{-1} u = G\alpha \quad \Longleftrightarrow \quad u = F_{\text{old}} G \alpha.
\end{equation}

Substitute into the objective and minimize over $\alpha$:
\begin{multline}
\min_{\alpha} \quad g^\top F_{\text{new}} g - 2g^\top F_{\text{old}} G \alpha \\
+ \alpha^\top G^\top F_{\text{old}} F_{\text{new}}^{-1} F_{\text{old}} G \alpha.
\end{multline}

Taking the derivative with respect to $\alpha$ and setting to zero:
\begin{align}
-2 G^\top F_{\text{old}} g + 2 G^\top F_{\text{old}} F_{\text{new}}^{-1} F_{\text{old}} G \alpha &= 0 \nonumber\\
\implies \alpha^* = (G^\top F_{\text{old}} F_{\text{new}}^{-1} F_{\text{old}} G)^{-1} G^\top F_{\text{old}} g.
\end{align}

The optimal $u^*$ is therefore
\begin{equation}
u^* = F_{\text{old}} G (G^\top F_{\text{old}} F_{\text{new}}^{-1} F_{\text{old}} G)^{-1} G^\top F_{\text{old}} g.
\end{equation}

Define the projection matrix
\begin{equation}
P = I - F_{\text{old}} G (G^\top F_{\text{old}} F_{\text{new}}^{-1} F_{\text{old}} G)^{-1} G^\top F_{\text{old}}.
\end{equation}

Then $g - u^* = P g$, and the unnormalized update is
\begin{equation}
\tilde{v} = c F_{\text{new}}^{-1}(g - u^*) = c F_{\text{new}}^{-1} P g.
\end{equation}

To enforce the trust region constraint $v^\top F_{\text{new}} v = \eta^2$, we compute
\begin{align}
\tilde{v}^\top F_{\text{new}} \tilde{v} &= c^2 (Pg)^\top F_{\text{new}}^{-1} F_{\text{new}} F_{\text{new}}^{-1} (Pg) \nonumber\\
&= c^2 g^\top P^\top F_{\text{new}}^{-1} P g.
\end{align}

Setting this equal to $\eta^2$ gives
\begin{equation}
c = \frac{\eta}{\sqrt{g^\top P^\top F_{\text{new}}^{-1} P g}}.
\end{equation}

The final update is
\begin{equation}
v^* = \eta \frac{F_{\text{new}}^{-1} P g}{\sqrt{g^\top P^\top F_{\text{new}}^{-1} P g}}.
\end{equation}
\end{proof}

\subsection{Proof of Theorem \ref{thm:fng}: Fisher Natural Gradient}
\label{sec:proof_fng}

\begin{proof}
We parameterize the update as $v = c F_{\mathrm{new}}^{-1} g$ and seek the value of $c$ satisfying the trust region constraint $v^\top F_{\mathrm{new}} v = \eta^2$.

Substituting the parameterization:
\begin{align}
v^\top F_{\mathrm{new}} v &= (c F_{\mathrm{new}}^{-1} g)^\top F_{\mathrm{new}} (c F_{\mathrm{new}}^{-1} g) \nonumber\\
&= c^2 g^\top F_{\mathrm{new}}^{-1} g.
\end{align}

Setting this equal to $\eta^2$:
\begin{equation}
c^2 g^\top F_{\mathrm{new}}^{-1} g = \eta^2.
\end{equation}

Solving for $c$:
\begin{equation}
c = \frac{\eta}{\sqrt{g^\top F_{\mathrm{new}}^{-1} g}}.
\end{equation}

Substituting back into the parameterization gives the FNG update:
\begin{equation}
v_{\mathrm{FNG}} = \eta \, \frac{F_{\mathrm{new}}^{-1} g}{\sqrt{g^\top F_{\mathrm{new}}^{-1} g}}.
\end{equation}
\end{proof}

\section{Hyperparameter Values}
\label{app:hyperparameters}
We list the hyperparameters used in our experiments for reproducibility in Table \ref{tab:hparams}. We performed a comprehensive sweep over all hyperparameters. For each task, we selected the hyperparameter configuration that yielded the highest final task accuracy on a validation set; however, for MNIST-based tasks we filtered out low learning rate runs that may yield high final task accuracy by not learning anything at all (these are also high variance and not representative of any meaningful learning). Hence, we only considered runs that achieved at least 90\% accuracy on the first task when trained on it.

For the learning rate, we swept over [1e-5, 5e-5, 1e-4, 5e-4, 1e-3, 5e-3, 1e-2, 5e-2, 1e-1]. For EWC, the $\lambda$ regularization coefficient was swept over [10, 50, 100, 400]. For FOPNG, FNG, FOPNG-PreFisher, the $\lambda$ regularization coefficient was swept over [1e-4, 5e-4, 1e-3, 1e-2] \textit{after} choosing the optimal learning rate while using $\lambda=$1e-3; it is likely that a more thorough sweep over FOPNG's hyper-parameters would yield slightly higher performance. Section \ref{app:ablations} shows that performance is extremely stable for different FOPNG $\alpha$ values, so we fix $\alpha = 0.5$ for all experiments.

\renewcommand{\arraystretch}{0.74}
\begin{table*}[htbp]  
\centering
\caption{Hyperparameters used for each method across different tasks.}
\label{tab:hparams}
\begin{tabular}{lccccccc}
\toprule
\textbf{Hyperparameter} 
& \textbf{Adam} 
& \textbf{SGD} 
& \textbf{EWC} 
& \textbf{FNG} 
& \textbf{OGD} 
& \textbf{FOPNG} 
& \textbf{FOPNG-PreFisher} \\
\midrule
\multicolumn{8}{l}{\textit{Permuted-MNIST}} \\
\midrule
Learning rate ($\eta$)  & 1e-4 & 5e-3 & 1e-2 & 1e-3 & 5e-3 & 1e-4 & 1e-4 \\
Grads per task          & -- & -- & -- & -- & 80 & 80 & 80 \\
Batch size              & 10 & 10 & 10 & 10 & 10 & 10 & 10 \\
$\lambda$ & -- & -- & 10 & 1e-3 & -- & 1e-2 & 1e-3 \\
$\alpha$ & -- & -- & -- & -- & -- & 0.5 & -- \\
Max directions & -- & -- & -- & -- & 400 & 400 & 400 \\
Fisher batch size & -- & -- & full & full & -- & full & full \\
\# epochs/task & 5 & 5 & 5 & 5 & 5 & 5 & 5 \\
\midrule
\multicolumn{8}{l}{\textit{Split-MNIST}} \\
\midrule
Learning rate ($\eta$)  & 1e-5 & 5e-4 & 5e-4 & 1e-3 & 5e-4 & 1e-5 & 1e-5 \\
Grads per task          & -- & -- & -- & -- & 80 & 80 & 80 \\
Batch size              & 10 & 10 & 10 & 10 & 10 & 10 & 10 \\
$\lambda$ & -- & -- & 400 & 1e-3 & -- & 5e-4 & 5e-4 \\
$\alpha$ & -- & -- & -- & -- & -- & 0.5 & -- \\
Max directions & -- & -- & -- & -- & 400 & 400 & 400 \\
Fisher batch size & -- & -- & full & full & -- & full & full \\
\# epochs/task & 5 & 5 & 5 & 5 & 5 & 5 & 5 \\
\midrule
\multicolumn{8}{l}{\textit{Rotated-MNIST}} \\
\midrule
Learning rate ($\eta$)  & 1e-4 & 1e-1 & 5e-4 & 5e-4 & 5e-4 & 5e-4 & 1e-3 \\
Grads per task          & -- & -- & -- & -- & 80 & 80 & 80 \\
Batch size              & 10 & 10 & 10 & 10 & 10 & 10 & 10 \\
$\lambda$ & -- & -- & 10 & 1e-3 & -- & 1e-2 & 1e-2 \\
$\alpha$ & -- & -- & -- & -- & -- & 0.5 & -- \\
Max directions & -- & -- & -- & -- & 400 & 400 & 400 \\
Fisher batch size & -- & -- & full & full & -- & full & full \\
\# epochs/task & 5 & 5 & 5 & 5 & 5 & 5 & 5 \\
\midrule
\multicolumn{8}{l}{\textit{Split-CIFAR10}} \\
\midrule
Learning rate ($\eta$)  & 1e-3 & 5e-2 & 1e-2 & 1e-2 & 5e-2 & 1e-3 & 1e-3 \\
Grads per task          & -- & -- & -- & -- & 80 & 80 & 80 \\
Batch size              & 10 & 10 & 10 & 10 & 10 & 10 & 10 \\
$\lambda$ & -- & -- & 50 & 1e-3 & -- & 1e-3 & 1e-3 \\
$\alpha$ & -- & -- & -- & -- & -- & 0.5 & -- \\
Max directions & -- & -- & -- & -- & 400 & 400 & 400 \\
Fisher batch size & -- & -- & 1024 & 1024 & -- & 1024 & 1024 \\
\# epochs/task & 5 & 5 & 5 & 5 & 5 & 5 & 5 \\
\midrule
\multicolumn{8}{l}{\textit{Split-CIFAR100}} \\
\midrule
Learning rate ($\eta$)  & 1e-4 & 5e-3 & 1e-2 & 5e-3 & 1e-2 & 5e-3 & 5e-3 \\
Grads per task          & -- & -- & -- & -- & 80 & 80 & 80 \\
Batch size              & 10 & 10 & 10 & 10 & 10 & 10 & 10 \\
$\lambda$ & -- & -- & 50 & 1e-3 & -- & 1e-3 & 1e-3 \\
$\alpha$ & -- & -- & -- & -- & -- & 0.5 & -- \\
Max directions & -- & -- & -- & -- & 800 & 800 & 800 \\
Fisher batch size & -- & -- & full & 1024 & -- & 1024 & 1024 \\
\# epochs/task & 10 & 10 & 10 & 10 & 10 & 10 & 10 \\
\bottomrule
\end{tabular}
\end{table*}

\section{Continued Ablation Studies}
\label{app:ablations}

We conduct ablation studies on FOPNG's two key hyperparameters: the regularization parameter $\lambda$ and the new Fisher weighting $\alpha$. We use the best hyperparameter settings from Table \ref{tab:hparams} for our runs.

\subsection{Regularization Parameter $\lambda$}

We conduct a sweep of $\lambda$ over [1e-5, 3e-5, 5e-5, 7e-5, 1e-4, 3e-4, 5e-4, 7e-4, 1e-3, 3e-3, 5e-3, 7e-3, 1e-2] and measure the final task accuracy for each task on Split-MNIST and Rotated-MNIST, keeping all other parameters fixed.

Figure~\ref{fig:lambda_ablation} shows that FOPNG performance is largely stable across a wide range of $\lambda$ values. In general we have found that $\lambda \in [10^{-4}, 10^{-3}]$ is consistently good for all tasks. Very small values can lead to numerical instability when the Fisher matrix is ill-conditioned, while very large values begin to perturb the Fisher geometry significantly, reducing the method's effectiveness. This stability is one key advantage of our method; unlike EWC, the additional hyperparameters introduced by FOPNG effectively require no extra tuning.

\begin{figure}[htbp]
    \centering
    \begin{minipage}{0.4\textwidth}
        \centering
        \includegraphics[width=\textwidth]{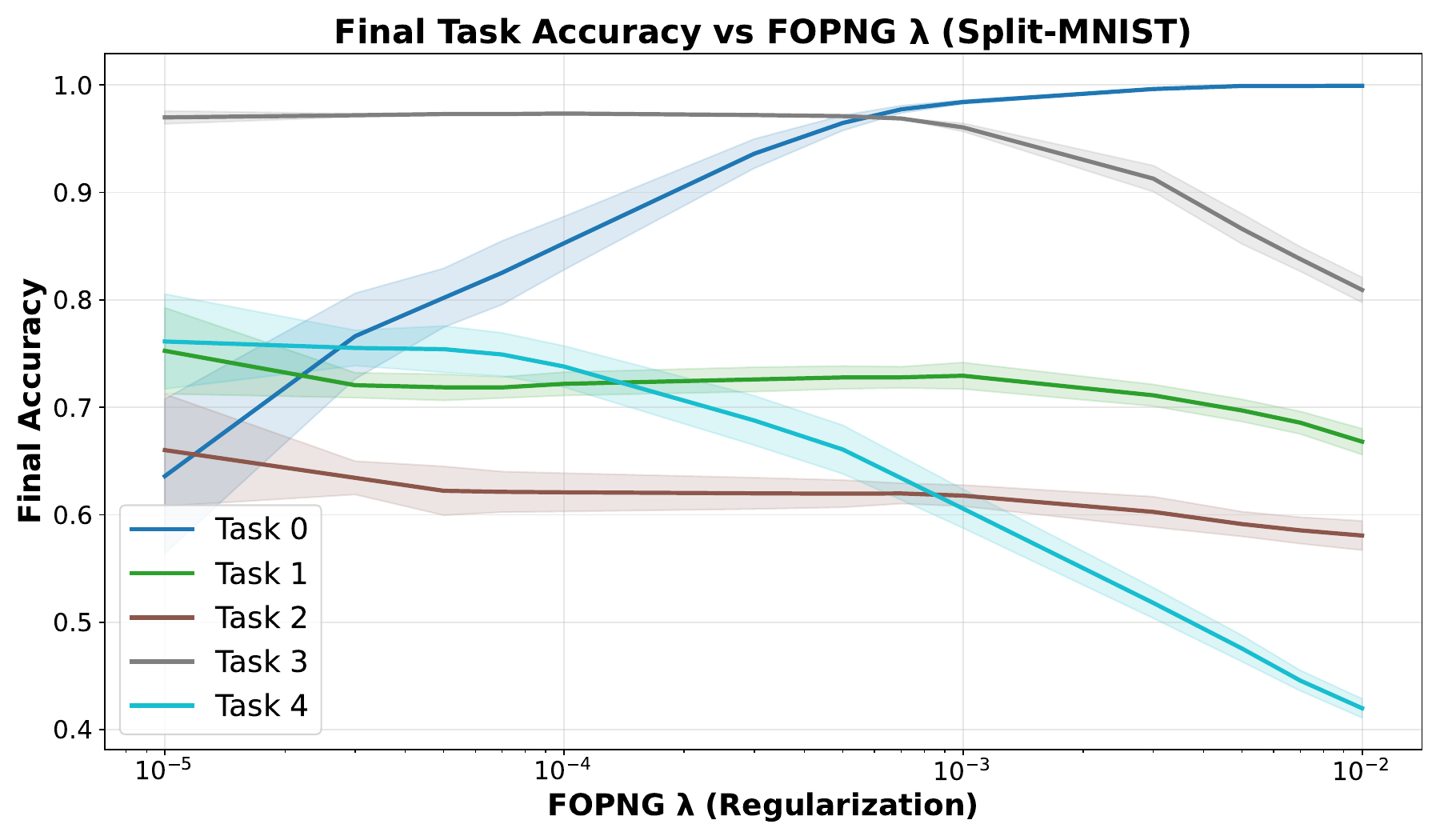}
    \end{minipage}
    \hfill
    \begin{minipage}{0.4\textwidth}
        \centering
        \includegraphics[width=\textwidth]{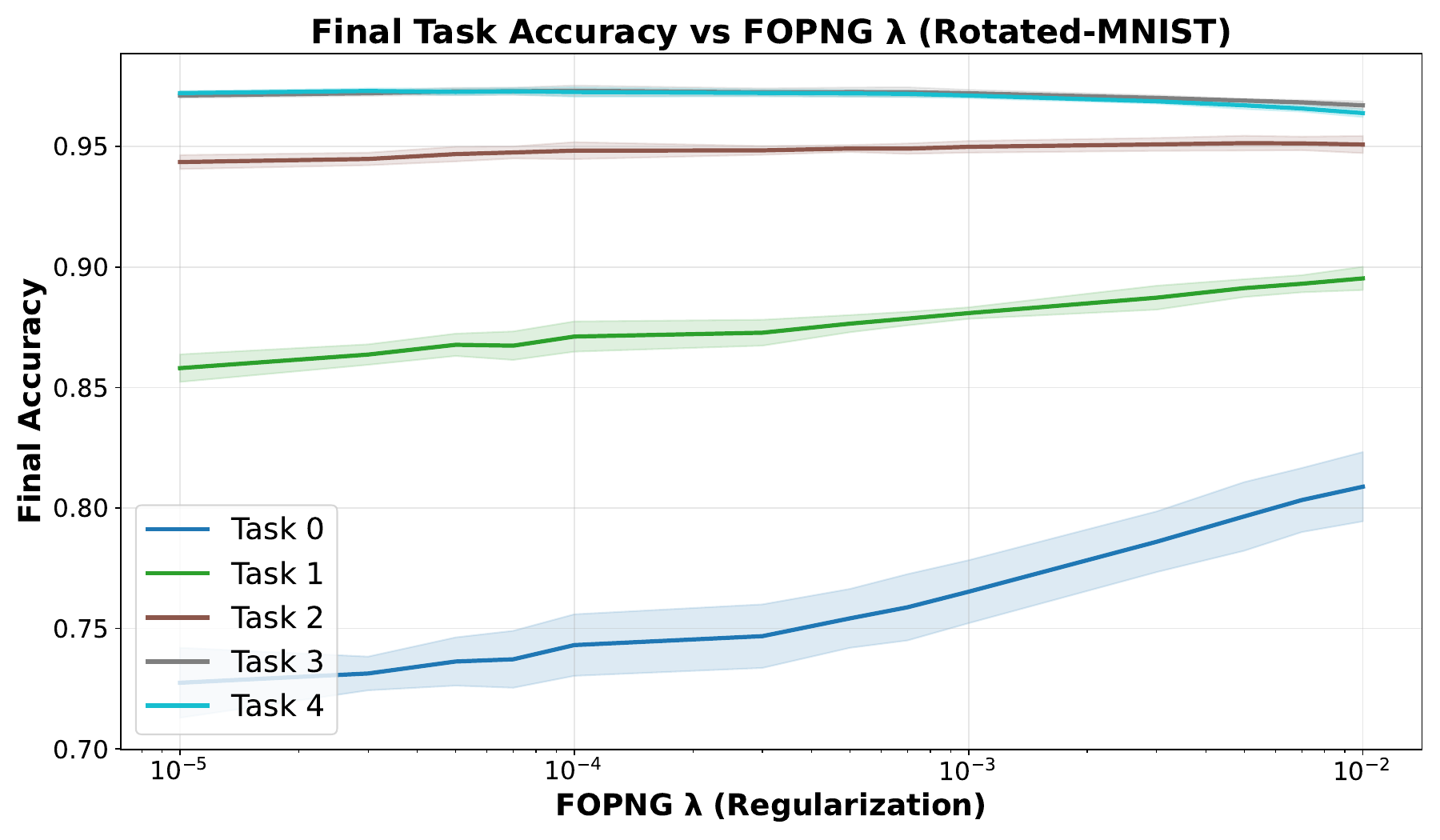}
    \end{minipage}
    \caption{Performance is generally stable in intermediate ranges of $\lambda$. Shaded regions show 95\% confidence intervals across 5 independent runs.}
    \label{fig:lambda_ablation}
\end{figure}

\subsection{New Fisher Weighting $\alpha$}

We conduct a sweep of $\alpha$ over [0.2, 0.3, 0.4, 0.5, 0.6, 0.7, 0.8, 0.9] and measure the final task accuracy for each task on Split-MNIST and Rotated-MNIST, keeping all other parameters fixed. Figure~\ref{fig:alpha_ablation} demonstrates that FOPNG is remarkably robust to the choice of $\alpha$. While we do not have a concrete explanation for this, one possible explanation is that the change in Fisher information between tasks is not significant enough to offset the information loss from exponential averaging.

\begin{figure*}[htbp]
    \centering
    \begin{minipage}{0.4\textwidth}
        \centering
        \includegraphics[width=\textwidth]{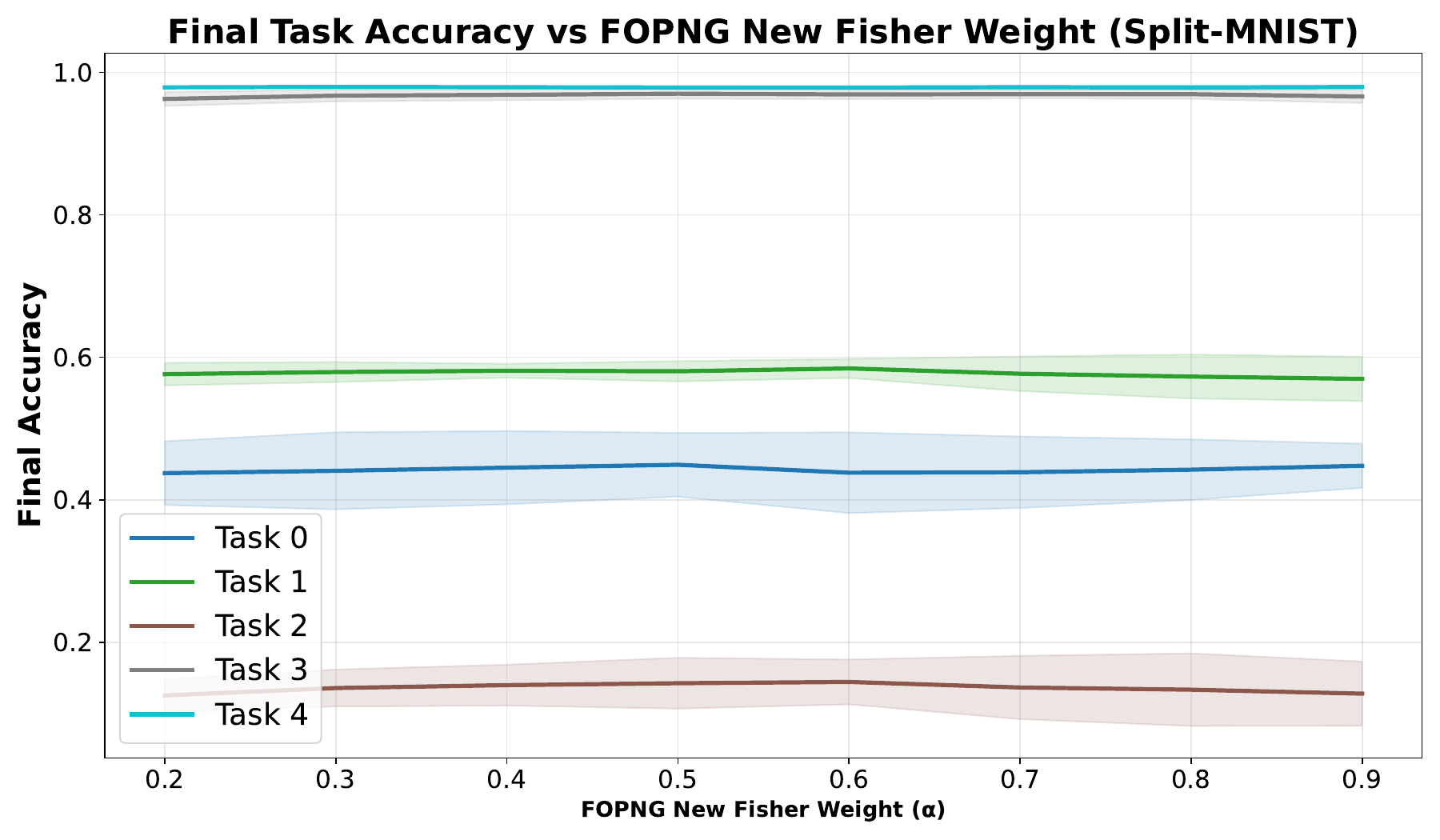}
    \end{minipage}
    \hfill
    \begin{minipage}{0.4\textwidth}
        \centering
        \includegraphics[width=\textwidth]{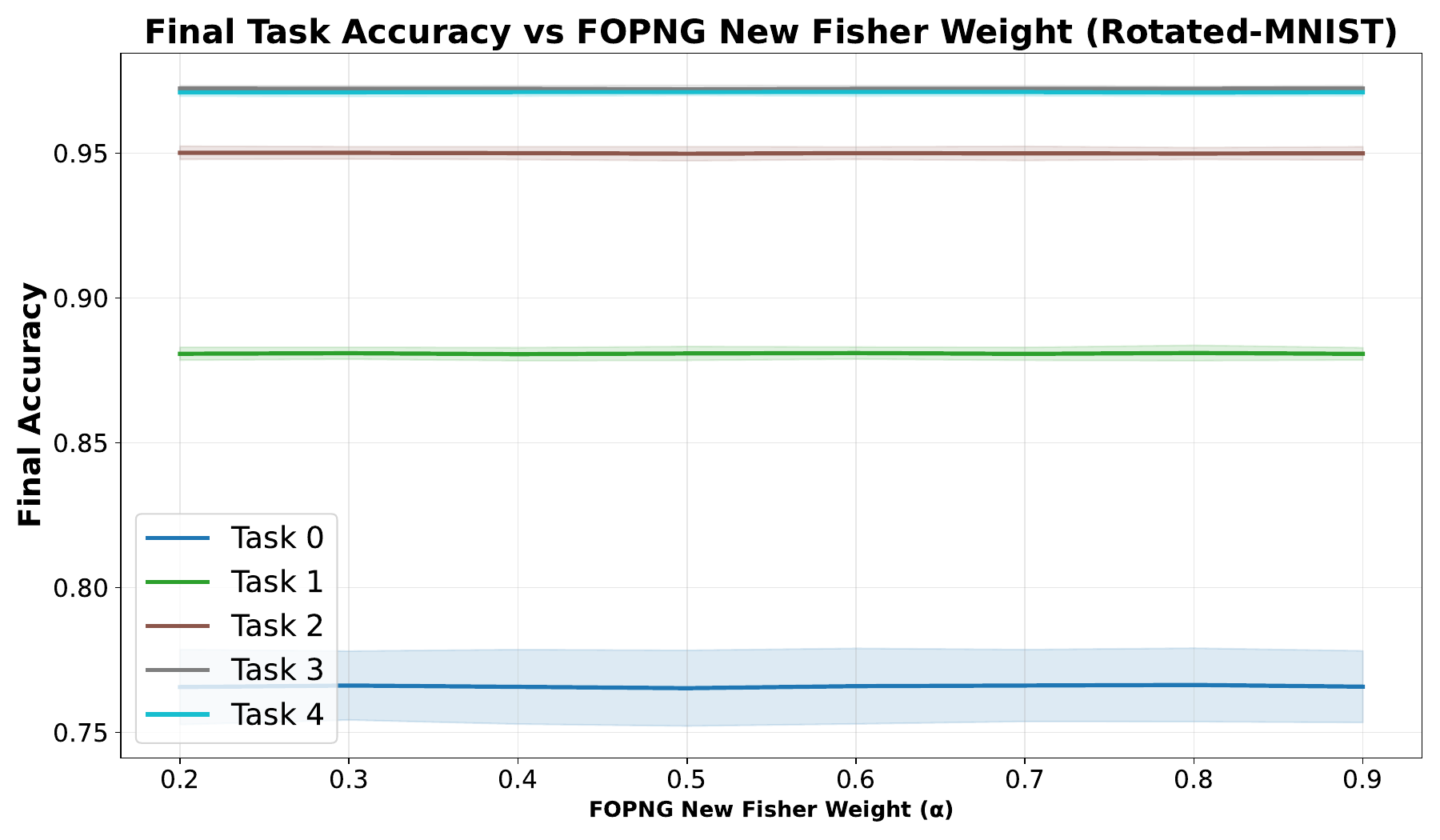}
    \end{minipage}
    \caption{Performance is extremely stable across all values of $\alpha$. Shaded regions show 95\% confidence intervals across 5 independent runs}
    \label{fig:alpha_ablation}
\end{figure*}

\end{document}